%% file: Template.tex
\documentclass{article}
\usepackage{spconf,graphicx,hyperref}
\usepackage{amsmath}
\usepackage{amssymb}
\usepackage{mathtools}
\usepackage{amsthm}
\usepackage{bm}
\usepackage{dsfont}
\usepackage{xcolor}
\usepackage{subcaption}

\usepackage{mdframed}
\usepackage{amsthm}
\newtheorem{theorem}{Theorem}

\newtheorem{assumption}{Assumption}

\newtheorem{remark}{Remark}
\newtheorem{proposition}{Proposition}

\usepackage{afterpage}

\usepackage[textsize=tiny]{todonotes}

\usetikzlibrary{arrows,shapes,calc,decorations.pathreplacing,intersections,quotes,angles,positioning,fit,petri,through}
\usepackage{tikz,pgfplots,filecontents}
    \pgfdeclarelayer{background}
    \pgfdeclarelayer{foreground}
    \pgfsetlayers{background,main,foreground}
\usetikzlibrary{spy}
\usetikzlibrary{shapes}
\usetikzlibrary{plotmarks}

\pgfmathdeclarefunction{gauss}{2}{%
   \pgfmathparse{1/(#2*sqrt(2*pi))*exp(-((x-#1)^2)/(2*#2^2))}%
}

\input{math_commands.tex}

\title{Concept Activation Vectors: a Unifying View and Adversarial Attacks}

\name{
  Ekkehard Schnoor $^{\dagger \ddagger}$, 
  Malik Tiomoko $^{\ast}$, 
  Jawher Said $^{\ddagger}$, 
  Alex Jung $^{\dagger}$, 
  Wojciech Samek $^{\ddagger\mathsection}$
  \thanks{This work was supported by the German Research Foundation (DFG) as research unit DeSBi [KI-FOR 5363] (459422098), by the Research Council of Finland (Decision \#363624) as \textit{A Mathematical Theory of Trustworthy Federated Learning (MATHFUL)}, 
  by the Jane and Aatos Erkko Foundation (Decision \#A835) as
  \textit{A Mathematical Theory of Federated Learning (TRUST-FELT)},
  and by Business Finland as
  \textit{Forward-Looking AI Governance in Banking \& Insurance (FLAIG)}.
  }
}

\address{
  $^{\dagger}$ Department of Computer Science, Aalto University, Espoo, Finland \\
  $^{\ddagger}$ Department of Artificial Intelligence, Fraunhofer Heinrich Hertz Institute, Berlin, Germany \\
  $^{\ast}$ Huawei Noah’s Ark Lab, Huawei Technologies, Paris, France\\
  $^{\mathsection}$ Department of Electrical Engineering and Computer Science, Technical University Berlin, Germany\\
}

\begin{document}

\maketitle


\begin{abstract}

Concept Activation Vectors (CAVs) are a tool from explainable AI, offering a promising approach for understanding how human-understandable concepts are encoded in a model's latent spaces. They are computed from hidden-layer activations of inputs belonging either to a concept class or to non-concept examples. Adopting a probabilistic perspective, the distribution of the (non-)concept inputs induces a distribution over the CAV, making it a random vector in the latent space. 
This enables us to derive mean and covariance for different types of CAVs, leading to a unified theoretical view.
This probabilistic perspective also reveals a potential vulnerability: CAVs can strongly depend on the rather arbitrary non-concept distribution, 
a factor largely overlooked in prior work. We illustrate this with a simple yet effective adversarial attack, underscoring the need for a more systematic study.

\end{abstract}

\vspace{-0.2cm}
\section{Introduction}\label{sec:intro}
\vspace{-0.2cm}

Many XAI techniques generate attribution maps, highlighting important input regions for individual predictions. 
However, alone they provide limited insight, as often they do not specify what exactly the model has identified in those regions.
Beyond mere feature-based explanations, \textit{concept-based} explanations aim for explanations in terms of human-understandable concepts. 
Among them are CAVs \cite{kim2018interpretability}, that perform linear probings to detect if human-understandable concepts  are encoded in latent layers, and the related \textit{Testing with Concept Activation Vectors (TCAV)}, that aims to quantify the contribution of a certain concept (e.g., \textit{stripes}) to the prediction of a class (e.g. \textit{zebras}). 
The original approach of \cite{kim2018interpretability} proposes classical linear classifiers like Support Vector Machines (SVMs), Ridge Regression or the LASSO  \cite{tibshirani1996regression} 
to derive a CAV, which have been compared and applied for unlearning biases in \cite{dreyer2024hope}. 
Other methods for computing CAVs include \textit{PatternCAV} \cite{pahde2022navigating}, and the recently introduced \textit{FastCAV} \cite{schmalwasser2025fastcav}, both of which we will analyse more closely in this paper.
While the aforementioned works consider applications in computer vision, CAVs were also applied successfully
e.g. to analyse the acquisition of chess concepts during training by self-play in \textit{AlphaZero} \cite{mcgrath2022acquisition, schut2023bridging}.
%
%
 %
Despite the successful applications of CAVs, a systematic understanding of their performance and robustness is still lacking. 
By construction, a CAV  corresponds to the normal vector of a linear separator that distinguishes between (non-)concept activations.\footnote{Note that e.g. also for \textit{PatternCAV} and \textit{FastCAV}, an (optimal) separating hyperplane can be determined; therefore, we can interpret them as classifiers, even though they are not based on a classical linear classifier.} 
This linear nature allows us to draw on the extensive body of work analyzing high-dimensional linear models with asymptotically sharp performance characterizations.
We focus on the ridge regression \cite[Chapter 2.3]{tiomoko2021advanced}, and refer to similar works on SVMs \cite{liao2019large}, logistic regression \cite{mai2019large}, LASSO \cite{tiomoko2022deciphering}.
The classification accuracy of a CAV-based classifier can be interpreted as a measure for the degree to which a neural network has encoded a concept. It is of interest also in the context of Concept Bottleneck Models \cite{pmlr-v119-koh20a}, which aim to increase the interpretability, typically at the price of a lower accuracy.
In this paper, we propose a unifying probabilistic framework to investigate CAVs, enabling a more rigorous comparison 
of different CAV types by their distribution and classification performance. Our main contributions are as follows. 

\vspace{-0.3cm}
\begin{itemize}  \setlength\itemsep{-0.2em}
    \item We derive mean and covariance of \textit{PatternCAV} \cite{pahde2022navigating} and \textit{FastCAV} \cite{schmalwasser2025fastcav} in terms of the hidden-layer statistics in Proposition \ref{prop:main_th_result}, 
    allowing to predict their accuracies.
    \item We reveal the equivalence of seemingly different CAV methods the case of balanced classes\footnote{This is not a strong restriction and typically the case in practice, as usually non-concept examples are cheap to obtain, making it easy to match the number of previously generated concept examples.} and large regularization parameter in the case of the ridge regression.
    \item We present an adversarial attack on the \textit{TCAV} method.
\end{itemize}

\vspace{-0.2cm}
\section{ASSUMPTIONS AND SETUP}\label{sec:setup_and_assumptions}
\vspace{-0.2cm}

Consider an $L$-layer neural network for classification, written as the concatenation $f_l \circ h_{l}$, where $f_l : \R^{d_0} \to \R^{d_l}$ is the mapping from the input (of size $d_0$) up to layer $l$ (of size $d_l$); similarly, $h_l : \R^{d_l} \to \R^{d_L}$ is the mapping from the $l$th to the final $L$th layer. It is convenient to restrict ourselves to a specific class prediction, \ie $h_{l,k}:  \R^{d_l} \to \R$ is  similar to $h_{l}$, but restricted to class $k$ at the output. 
Assume we are given a set of examples collected by a user to illustrate some concept $C$ of interest, as well as a set of
\textit{non-concept} (or ``random'') examples (arbitrary input examples not containing the \textit{concept} $C$, or noise).
We may then train a linear classifier to separate their $l$th layer activations, and in this way obtain a CAV $\vv_C^l \in \R^{d_l}$ that is orthogonal to the separating hyperplane (showing towards the concept region by convention). 
A high accuracy (of the CAV classifier) may be interpreted in the sense that the neural network has encoded the concept.
While motivated by CAVs, we may often use a convenient compact notation considering a generic binary classification 
problem with a random training data matrix $\mX = [\vx_1,\ldots,\vx_n] \in \R^{d\times n}$ of $n$ i.i.d. data points of feature size $d$ and its associated label vector $\vy = [y_1,\ldots,y_n] \in \{-1,1\}^d$ drawn from a mixture distribution of two classes with class-specific (conditional) mean $\vmu_\ell \in \R^d$ and covariance $\mSigma_\ell \in \R^{d \times d}$ for $\ell = 1,2$, and satisfying Assumption \ref{assum:concentration}, for the $n_{\ell}$ vectors of class $\mathcal C_\ell$, $\ell \in \{1,2\}$; in particular, $n = n_1+n_2$. Our goal is to predict the label $y$ for a new test datum $\vx \in \R^d$. 

\begin{assumption}[Data Concentration]
\label{assum:concentration}
The random vector $\vx \in \R^d$ is \emph{$q$-exponentially concentrated}, i.e., for any $1$-Lipschitz continuous (w.r.t. the $\ell_2$-norm) function 
$\varphi : \R^d \to \R$ it holds
\[
\P\left(|\varphi(\vx_i) - \E[\varphi(\vx_i)]| \ge t\right)
\le C e^{-(t/\sigma)^q} \qquad \forall t>0,
\]
for some constants $q>0$, $C>0$, $\sigma>0$ independent of $d$. 
\end{assumption}

Random vectors satisfying Assumption \ref{assum:concentration} include isotropic Gaussian random vectors, the uniform distribution on the sphere, and notably any Lipschitz-continuous transformations thereof (e.g., features from GANs~\cite{seddik2020random}). As (layers of) neural networks constitute Lipschitz-continuous functions, this framework is 
highly suitable in the context of CAVs.

We will use a result \cite[Chapter 2.3]{tiomoko2021advanced} for the ridge regression which makes use of an ``large $n$, large $d$'' assumption for its theoretical derivations, \ie formally assuming that $n > d$ and, as $n_\ell, n, d \to \infty$, asymptotically $d/n \to c_0 \in (0,1)$ and  ${n_{\ell}}/n\to c_{\ell} > 0$ for $\ell = 1, 2$, with class probabilities $c_1, c_2$; often we will just consider the balanced case of $n_1 = n_2$ and  therefore $c_1 = c_2 = 1/2$ with $c_1, c_2 \in (0,1)$, $c_1 + c_2 = 1$. 
Still, the asymptotic prediction is often accurate in a finite but large-dimensional setting.
Other results, notably our main result Proposition \ref{prop:main_th_result}, hold in a non-asymptotic setting.

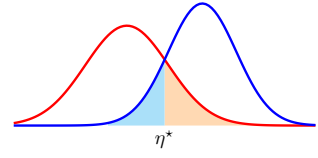
\begin{figure}[h]
\centering
\begin{tikzpicture}[scale=0.75]
    \begin{axis}[no markers, domain=-4:4, samples=100, axis lines*=left, height=5cm, width=8cm,
                 xtick=\empty, ytick=\empty, axis line style={draw=none}, ymin=-0.2, ymax=0.5] 
                     \addplot [fill=orange!30, draw=none, domain=0:+2.5] {gauss(-1,1.1)} \closedcycle ;
                     \addplot [fill=cyan!30, draw=none, domain=-1.5:0] {gauss(1,0.9)} \closedcycle ;
                     \addplot [very thick,red!100!black] {gauss(-1,1.1)};
                     \addplot [very thick,blue!100!black] {gauss(+1,0.9)};
                     \node at (axis cs: 0, -0.05) {$\eta^\star$};  
    \end{axis}
\end{tikzpicture}
\vspace{-0.2cm}
\caption[Gaussian classification scores (illustration)]
{Illust. of normal distributions of
$g(\vx) = \vx^\top \vw$ for \textcolor{red}{$\vx \in \mathcal{C}_1$ (red)} and \textcolor{blue}{$\vx \in \mathcal{C}_2$ (blue)}
with optimal decision threshold $\eta^\star$ at the intersection of the density function between the means.
}
\label{fig:two_gaussians}
\end{figure}

As CAVs perform a linear probing, we consider a linear classifier with weight vector $\vw \in \R^d$ and bias $\eta \in \R$,
\begin{equation}
    g(\vx) = \frac{1}{\sqrt{n}} {\vw}^\top\vx  \underset{\mathcal{C}_1}{\overset{\mathcal{C}_2}{\gtrless}} \eta.
    \label{eq:general_linear_classifier_g}
\end{equation}
\noindent Let us recall the following result, which enables us to predict in advance classification accuracies of linear classifiers, assuming class-specific normal distributions 
of the classification scores $g(\vx)$, which can be decomposed in terms of the data first and second-order statistics. We provide a generic statement;
in the CAV context, the data corresponds to the hidden-layer activations of the \textit{non-concept} (class $\mathcal{C}_1$, label $-1$) \textit{vs.} the \textit{concept} 
(class $\mathcal{C}_2$, label $1$) examples. Note that this result has been used before, e.g. in \cite{tiomoko2022deciphering}; still, to our best knowledge it has never been explored in the context of CAVs.

\begin{theorem}\label{eq:thm_classification_accuracy}
Let $\vw \in \R^d$ be the (random) weight vector of the linear model $g$ in \eqref{eq:general_linear_classifier_g},
with mean $\bar{\vw} = \E[\vw]$ and covariance $\mSigma_\vw = \Cov(\vw)$. 
Assume $g(\vx)$ to
have class-specific normal distributions, \ie $g(\vx) \sim \mathcal{N}\left(m_\ell, \sigma_\ell^2\right)$ for (a test data point) of either class, \ie
$ \vx \in \mathcal{C}_\ell$ independent of $\vw$, for $\ell=1,2$.
Then, for $ \vx \in \mathcal{C}_\ell$ with mean $\vmu_\ell$  and covariance $\mSigma_\ell$ it holds  
\begin{align*}
    m_\ell
= & \frac{1}{\sqrt{n}} \E_{\vw,\vx}\left[{\vw}^\top \vx\right] 
=   \frac{1}{\sqrt{n}} \bar{\vw}^\top \bar{\vx}, \\
    \sigma_\ell^2 
= & \frac{1}{n}\tr\left(\mSigma_\vw\mSigma_\vx\right) 
    +   \frac{1}{n}\tr\left(\mSigma_\vw \bar{\vx}\bar{\vx}^\top\right)  +   \frac{1}{n}\tr\left(\mSigma_\vx \bar{\vw}\bar{\vw}^\top\right). 
\end{align*}
\noindent
Consequently, the (optimal) classification accuracy of \eqref{eq:general_linear_classifier_g} is given by 
$1-\varepsilon$, where the probability of misclassification $  \epsilon $ is  
\begin{align}
\label{eq:error_of_misclassification}
    \epsilon 
=   &   c_1 \cdot \textcolor{orange}{\P \left( \vx \rightarrow \mathcal{C}_2 \, | \, \vx \in \mathcal{C}_1 \right)} + 
        c_2 \cdot \textcolor{cyan}{\P \left( \vx \rightarrow \mathcal{C}_1 \, | \, \vx \in \mathcal{C}_2 \right) } \nonumber \\
=   &   c_1 \cdot \textcolor{orange}{\P \left(X > \eta \, | \, X \sim \mathcal{N}\left(m_1, \sigma_1^2\right) \right)} \; + \nonumber \\
    &   c_2 \cdot \textcolor{cyan}{\P \left( Y < \eta  \, | \, Y \sim \mathcal{N}\left(m_2, \sigma_2^2\right) \right)} \; ,
\end{align}
where $\eta\in \R$ denotes the (optimal) decision threshold in \eqref{eq:general_linear_classifier_g}.
\end{theorem}

\begin{proof}
The formulas for $ m_\ell$ and $\sigma_\ell^2$ follow from a straightforward computation of the mean and variance, 
by the independence of $\vw$ and $\vx$ (and even hold without the assumption of Gaussianity).
The formula for $\varepsilon$ is due to the classification rule \eqref{eq:general_linear_classifier_g} and the assumption of $g(\vx)$ to be Gaussian.
\end{proof}

The two summands adding up to $\varepsilon$ in \eqref{eq:error_of_misclassification} correspond to the areas of respective color in Fig. \ref{fig:two_gaussians}.
Notably, the classification accuracy only depends on the few scalars $m_1, m_2, \sigma_1^2, \sigma_2^2$, from which it can be computed by numerical integration of the normal distribution's CDF.
To obtain those scalars, we need to estimate both the data as well as the model's means and covariances.
While in practice both are simply estimated empirically from the training data, 
deriving the induced distribution over $\vw$ is challenging (see Sec. \ref{sec:main_results}).
Finally, we remark that the assumption of normal distributions of $g(\vx)$ in \eqref{eq:general_linear_classifier_g} is highly realistic under Assumption \ref{assum:concentration} and has been often observed in practice; see \cite[Theorem 3.2]{seddik2021unexpected} and \cite{klartag2007central, fleury2007stability} for a version of the central limit theorem for concentrated random vectors.
\noindent
Next, let us recall a few different methods to compute CAVs, preparing for our main results in the next section,
which includes an application of Theorem \ref{eq:thm_classification_accuracy} for the different ways to compute the CAV, corresponding to the weight vector $\vw$.
One way to obtain the CAV $\vw$ is by the \textit{ridge regression}
\begin{equation}
    \vw_\ridge = \argmin_{\vw \in \R^d} \left\| \vy - \frac{\,\mX^\top}{\sqrt{n}} \vw \right\|_2^2 + \lambda \| \vw \|_2^2.
  \label{eq:ridge_regression_problem}    
\end{equation}
We solve \eqref{eq:ridge_regression_problem} by \cite[Chapter 2.3]{tiomoko2021advanced} - see also \cite{cherkaoui2025high} - , iteratively computing $m_\ell, \sigma_\ell^2$ with $\vw = \vw_\ridge$, for $\ell=1,2$.  
As an alternative method to compute CAVs, the \textit{PatternCAV} \cite{pahde2022navigating} is
\begin{align}
    \vw_{\pat}
=  & \argmin_{\vw \in \R^{d}} \left\| \mX^\top - \vy \vw^\top \right\|_2^2   \\
=  & \frac{1}{n_2} \sum_{i=1}^{n_2} \vx_i^{(2)} - \frac{1}{n_1} \sum_{i=1}^{n_1} \vx_i^{(1)},
   \label{eq:pattern_CAV_loss_function}
\end{align}
showing in the direction of $\vmu_2$, \textit{towards the concept region}. 
For the recently proposed \textit{FastCAV} \cite{schmalwasser2025fastcav} we begin by computing
the joint empirical mean of both classes, denoted by $\hat{\vmu}_{1,2}$,

\begin{equation}
    \hat{\vmu}_{1,2} = \frac{1}{n_1 + n_2} \sum_{\ell=1}^{2} \sum_{i=1}^{n_i} \vx_i^{(\ell)}   \in \R^d.
    \label{eq:fastCAV_joint_empirical_mean_of_classes}
\end{equation} 
\noindent
The \textit{FastCAV} is the vector pointing from the global mean $\hat{\vmu}_{1,2}$ towards the mean of the concept class, \ie $\mathcal{C}_2$. 
Formally, it is

\begin{equation}
    \vw_{\fast} 
=   \frac{1}{n_2} \sum_{i=1}^{n_2} \left( \vx_i^{(2)} - \hat{\vmu}_{1,2} \right) 
=   \frac{1}{n_2} \sum_{i=1}^{n_2}  \vx_i^{(2)} - \hat{\vmu}_{1,2}. \label{eq:fastCAV_definition}    
\end{equation}

\vspace{-0.2cm}
\section{MAIN RESULTS AND EXPERIMENTS}\label{sec:main_results}
\vspace{-0.2cm}

We present our theoretical main results, together with experiments\footnote{We provide code to reproduce our experiments in this Github repository:\\ \url{github.com/jawhar00/cav-unifying-view-adversarial}} on simple synthetic data like Gaussian Mixture Models (GMMs) for binary classification, as well as the \textit{ResNet-18} model applied to the \textit{CIFAR-10} dataset, and a simple neural network with three hidden layers for times series classification trained by ourselves, with the time series input given by
\begin{equation}
y(t) = A \cdot \sin(2\pi \cdot f \cdot t) + T \cdot t + \epsilon, \quad t = 0, \dots, T,
\label{eq:time_series_model}
\end{equation}
with amplitude $A$, frequency $f$, trend $T$ and noise $\epsilon$. The classes are characterized by (a combination of) high or low values for $A$, $f$ and $T$ (presence or absence of the concepts).

\vspace{-0.1cm}
\subsection{Classification Accuracies of CAVs}
\vspace{-0.2cm}


We adopt a probabilistic standpoint, interpreting CAVs as random vectors in latent space whose distribution is induced by the non-linear transformations of the (non-)concept input distributions. This provides new insights into recently suggested techniques for computing CAVs. While relying on \cite[Chapter 2.3]{tiomoko2021advanced} for ridge regression \eqref{eq:ridge_regression_problem}, our main theoretical result derives the distributions of \textit{PatternCAV} \eqref{eq:pattern_CAV_loss_function} and \textit{FastCAV} \eqref{eq:fastCAV_definition}, revealing their performances and surprisingly close relation.

\begin{proposition}\label{prop:main_th_result}
[Distribution of ${\vw}_{\fast}$ and ${\vw}_{\pat}$]
\label{prop:hyperplane}
Mean and covariance of ${\vw}_{\fast}$ and ${\vw}_{\pat}$ can be compactly expressed in terms of the data distribution as follows.
Assuming class-specific means $\vmu_\ell$ and covariances $\mSigma_\ell$ for $\ell=1,2$, it holds
\begin{align}
 \E[  \vw_{\pat} ] & = \bar{\vw}_{\pat} = \vmu_2 - \vmu_1 \in \R^d, \label{eq:w_pat_mean} \\
 \Cov(\vw_{\pat})  & = \mSigma_{\vw_{\pat}} 
=   \frac{1}{n_1} \mSigma_1 + \frac{1}{n_2} \mSigma_2 \in \R^{d \times d}, \label{eq:w_pat_covariance}
\end{align}
For $n_1 = n_2$, $\vw_{\fast}$ is a scaled version\footnote{More generally, in the unbalanced case, $\vw_{\fast} = n_1/(n_1+n_2) \vw_{\pat}$.}
of $ \vw_{\pat}$ as $ \vw_{\fast} = \frac{1}{2}  \vw_{\pat}$.
Consequently, the mean and covariance of $ \bar{\vw}_{\fast} $ are 
\begin{align}
   \E [ \bar{\vw}_{\fast} ]
= &  \bar{\vw}_{\fast}  = \frac{\vmu_2 - \vmu_1}{2} \in \R^d, \label{eq:fastCAV_mean}    \\ 
     \Cov( \bar{\vw}_{\fast} )
=  & \mSigma_{\vw_{\fast}}  
=   \frac{1}{4 n_1}\mSigma_1 + \frac{1}{4 n_2}\mSigma_2 \in \R^{d \times d}. \label{eq:fastCAV_cov_definition}
 \end{align}
\end{proposition}

\begin{proof}
The case of the mean \eqref{eq:w_pat_mean} is obvious, and \eqref{eq:w_pat_covariance} concerning the covariance follows from a straightforward computation starting from \eqref{eq:pattern_CAV_loss_function}. Next, we then obtain the scaling,
\begin{equation*}
          \vw_{\fast} 
=   \frac{1}{n_2} \sum_{i=1}^{n_2}  \vx_i^{(2)} - 
    \frac{1}{2 n_1} \sum_{i=1}^{n_1} \vx_i^{(1)}   -
    \frac{1}{2 n_2} \sum_{i=1}^{n_2} \vx_i^{(2)}  
=   \frac{1}{2}  \vw_{\pat}, 
\end{equation*}
only by \eqref{eq:pattern_CAV_loss_function} and \eqref{eq:fastCAV_definition},
immediately yielding \eqref{eq:fastCAV_mean} and \eqref{eq:fastCAV_cov_definition}.
\end{proof}

\begin{remark}[Relationship between ${\vw}_{\fast}$, ${\vw}_{\pat}$ and ${\vw}_{\ridge}$]
It turns out that $ \vw_{\pat}$, and by $ \vw_{\fast} = \tfrac{1}{2}  \vw_{\pat}$, also $ \vw_{\fast} $, are closely related to the case of ridge regression, since for sufficiently large $\lambda = O(\sqrt{n})$, by simply inserting $\lambda = \sqrt{n}$,
\begin{align}
        \vw_{\ridge}
& =       \left(\frac{1}{n} \mX \mX^\top + \lambda \mI_p \right)^{-1} \frac{\mX}{\sqrt{n}} \vy  
\approx \frac{1}{\lambda}  \mI_p  \frac{\mX}{\sqrt{n}} \vy \label{eq:ridge_cav_pattern_cav_scaling} \\
& =   \frac{1}{n} \sum_{i=1}^n y_i \vx_i
=   \frac{1}{n} \left(\sum_{i=1}^{n_2} \vx_i^{(2)} - \sum_{i=1}^{n_1} \vx_i^{(1)} \right)
= \frac{1}{2} \vw_\pat . \nonumber
\end{align} 
\end{remark}
\vspace{-0.1cm}

Even though the hyperplane is the primary object of interest in \textit{TCAV}, Proposition \ref{prop:hyperplane} further allows us to derive the associated classification accuracy, which, while not essential for the computation of \textit{TCAV} itself, provides valuable theoretical insight into the reliability of concept encoding.
As the error $\epsilon$ from \eqref{eq:error_of_misclassification} remains unchanged when scaling $\vw$ (assuming $g(\vx)$ is normally distributed), $\vw_\pat$ and
$\vw_\fast$ turn out to have exactly the same accuracies as classifiers, and a performance similar to ridge regression for sufficiently large $\lambda$. This is confirmed by the experiments on synthetic data in Fig. \ref{fig:rr_pattern_fast_theory_emp_synthetic}.
Fig. \ref{fig:time_series_ridge_regression_lambda} is similar to Fig. \ref{fig:rr_pattern_fast_theory_emp_synthetic}, 
with ridge regression applied to the activations of the (non-)concept examples in the time series model, comparing errors for different layers.
Fig. \ref{fig:gaussian_distributions_ResNet18} is similar to the illustration in Fig. \ref{fig:two_gaussians}, for different
layers of the \textit{ResNet-18} model \cite{he2016deep} and \textit{CIFAR-10} dataset \cite{krizhevsky2009learning}, comparing theoretical prediction and empirical findings on the test set.

 \vspace{-0.3cm}

\begin{figure}[h]
    \centering
    \begin{subfigure}[b]{0.49\linewidth}
        \centering
        \includegraphics[width=\linewidth]{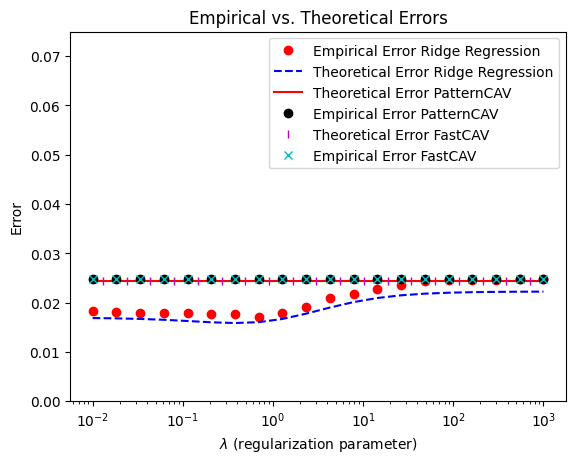}
        \caption{Errors of \textit{ridge regression} (plotted over $\lambda$), \textit{FastCAV} and \textit{PatternCAV} (constant, repeated for $\lambda$) for synthetic data (GMM).}
        \label{fig:rr_pattern_fast_theory_emp_synthetic}
    \end{subfigure}
    \hfill
    \begin{subfigure}[b]{0.49\linewidth}
        \centering
        \includegraphics[width=\linewidth]{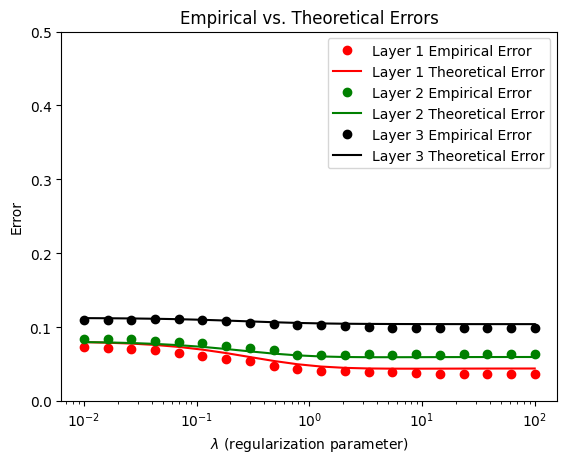}
        \caption{Time series \eqref{eq:time_series_model}: concept \textit{high freq.} vs. non-concept \textit{noise} $\vx \sim \mathcal{N}(\bm{0}, \mI)$ (at input); $\varepsilon$ at their latent layers activations (colors).}
        \label{fig:time_series_ridge_regression_lambda}
    \end{subfigure}
   \caption{Accurate theoretical prediction of the classification error $\varepsilon$ from \eqref{eq:error_of_misclassification} as a function of
            $\lambda$ (\textit{ridge regression}); different colors refer to different methods \ref{fig:rr_pattern_fast_theory_emp_synthetic} or different layers \ref{fig:time_series_ridge_regression_lambda}.}
    \label{fig:combined_figures}
\end{figure}



\begin{figure}[h]
    \centering
    \includegraphics[width=0.8\linewidth]{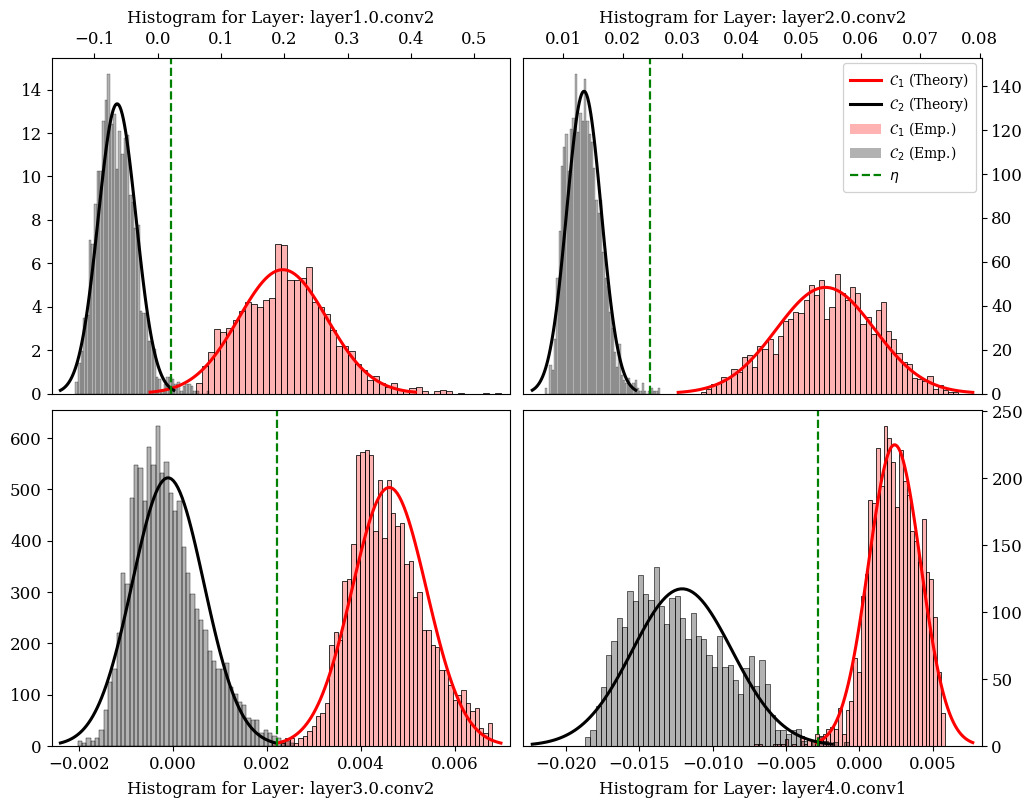}
    \caption{Close match between test-set histograms and the Gaussian predictions of the CAV projection for the  concept $\textit{blue}$ vs. \textit{noise} $\vx \sim \mathcal{N}(\bm{0}, \mI)$; ridge-regression (fixed $\lambda$), for different layers (\textit{ResNet-18} model for the \textit{CIFAR-10} dataset).}
    \label{fig:gaussian_distributions_ResNet18}
\end{figure}


\vspace{-0.8cm}
\subsection{TCAV Scores and Adversarial Attacks}
\vspace{-0.3cm}

To quantify the contribution of a concept $C$ to the prediction of class $\mathcal{C}_k$ for an input example $\vx \in \R^{d_0}$ based on
a CAV $\vv_C^l$ at layer $l$, a sensitivity score $ S_{C,k,l} (\vx)$ has been proposed \cite{kim2018interpretability},
\begin{align}
    S_{C,k,l} (\vx)
= & \lim_{{h \to 0}} \frac{h_{l,k} \left( f^l (\vx) + h \cdot  \vv_C^l \right) - h_{l,k} \left( f^l (\vx)\right)}{h} \nonumber \\
= & \left \langle  \nabla h_{l,k} \left( f^l (\vx)\right), \vv_C^l \right \rangle , \label{eq:TCAV_sensitivity_inner_product}
\end{align}
\ie the directional derivative of $h_{l,k}$ in direction of $\vv_C^l$,
equivalent to \eqref{eq:TCAV_sensitivity_inner_product} under sufficient smoothness assumptions.
To quantify the overall contribution of concept $C$ to class $\mathcal{C}_k$, \cite{kim2018interpretability} introduced
a score $\operatorname{TCAV_Q}_{C,k,l} \in [0,1]$, measuring the percentage of $S_{C,k,l} (\vx)$ being positive
for $\vx \in \mathcal{C}_k$
, or formally,
\begin{equation*}
    \operatorname{TCAV_Q}_{C,k,l} = \frac{| \{ \vx \in \mathcal{X}_k \, : \,  S_{C,k,l} (\vx) > 0 \}|}{|\mathcal{X}_k|} \; \in [0,1].
\end{equation*}
Note that $S_{C,k,l} (\vx)$ and $\operatorname{TCAV_Q}_{C,k,l}$ may strongly depend on the specific choice
of $\vv_C^l$, which in turn depends on the method to compute it, and on the rather arbitrary choice of the non-concept data distribution.
In \cite{kim2018interpretability}, randomized CAVs were computed repetitively. 
We aim to initiate a more systematic investigation, starting with the worst-case scenario by proposing an adversarial attack.
\begin{figure}
    \centering
    \includegraphics[width=1\linewidth]{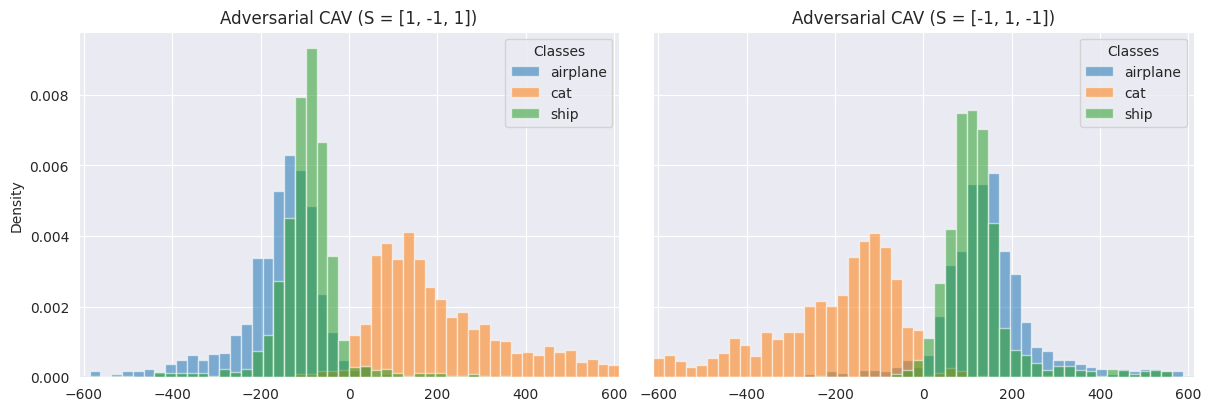}
    \caption{Histograms of 
    $S_{C,k,l} (\vx)$ from \eqref{eq:TCAV_sensitivity_inner_product} with $\vx$ from $3$ classes \textit{(colors)}
    at layer \textbf{layer3.0.conv2} of the \textit{ResNet-18} model for the \textit{CIFAR-10} dataset;
    using two different choices of $S = (s_1,s_2,s_3)$ to manipulate the CAV, and thus the histograms of the 
    scores $S_{C,k,l} (\vx)$, as well as $\operatorname{TCAV_Q}_{C,k,l}$.
}
    \label{fig:adversarial_plot_one}
\end{figure}
For a $K$-class classification problem with class-wise data matrices $\mX^{(k)}$, corresponding to the hidden-layer activations in the CAV context, consider a sign $s_k = \pm 1$. This may or may not be the ``correct'' sign for $\mathcal{C}_k$, but the sign we want to \textit{avoid} in an adversarial setting, where we would like to ``push the sensitivity scores of $\mX^{(k)}$ away from $s_k$'' (driving $\operatorname{TCAV_Q}_{C,k,l}$ either closer to $0$ or $1$), \ie we would like to enforce that $s_k \cdot \mX^{(k)} \vw< 0$  ``for many entries''. 
Precisely, we propose to minimize 
$\sum_{k=1}^K \mathds{1}_{\{s_k \cdot \mX^{(k)} \vw > 0\}}$,
where $\mathds{1}_{\{\vz > 0\}} := \sum_{i=1}^p \mathds{1}_{\{z_i > 0\}}$ counts the number of positive entries of $\vz \in \R^p$; for the implementation, we use a smooth approximation by the sigmoid function $\sigma(x) = 1/(1 + e^{-x})$, as illustrated in Fig. \ref{fig:adversarial_plot_one}.

\vspace{-0.2cm}
\section{CONCLUSION AND OUTLOOK}
\vspace{-0.2cm}

\label{sec:conclusion}



We described the distributions of \textit{PatternCAV} and \textit{FastCAV}, enabling to predict their classification accuracy, 
revealing their close connection to each other and to ridge regression.
We proposed an adversarial attack at the latent space as a worst-case scenario that manipulates explanations, 
leaving a general investigation of CAV's robustness for future work.

\bibliographystyle{IEEEbib}
\bibliography{strings,refs}

\end{document}

%% file: math_commands.tex
\def\R{\mathbb R}

\def\E{\mathbb E}
\def\P{\mathbb P}

\def\Cov{\mathrm{Cov}}
\renewcommand\epsilon{\varepsilon}

\newcommand{\tr}{{\rm tr}}

\DeclareMathOperator*{\argmin}{arg\,min}

\def\pat{{\operatorname{pat}}}
\def\fast{{\operatorname{fast}}}
\def\ridge{{\operatorname{ridge}}}

\def\vv{{\mathbf{v}}}
\def\vw{{\mathbf{w}}}
\def\vx{{\mathbf{x}}}
\def\vy{{\mathbf{y}}}
\def\vz{{\mathbf{z}}}

\def \vmu{{\bm{\mu}}}

\def\mI{{\mathbf{I}}}

\def\mX{{\mathbf{X}}}

\def\mSigma {{\mathbf{\Sigma }}}

\def\R{\mathbb{R}}
\def\P{\mathbb{P}}

\def\E{\mathbb{E}}

\newcommand{\ie}{\emph{i.e.,}~}